\documentclass{article} 
\usepackage{nips11submit_e,times}

\usepackage{latexsym,bm,amsmath, amssymb, amsthm, bbm, epsf, graphicx}

\title{Learning unbelievable marginal probabilities}

\author{
Xaq~Pitkow\\
Department of Brain and Cognitive Science\\
University of Rochester\\
Rochester, NY 14607 \\
\texttt{xaq@post.harvard.edu} \\
\And
Yashar~Ahmadian \\
Center for Theoretical Neuroscience \\
Columbia~University \\
New York, NY 10032 \\
\texttt{ya2005@columbia.edu} \\
\AND
Ken~D.~Miller \\
Center for Theoretical Neuroscience \\
Columbia~University \\
New York, NY 10032 \\
\texttt{ken@neurotheory.columbia.edu} \\
}

\newcommand{\la}{\left\langle}
\newcommand{\ra}{\right\rangle}

\newcommand{\be}{\begin{equation}}
\newcommand{\ee}{\end{equation}}
\newcommand{\bea}{\begin{eqnarray}}
\newcommand{\eea}{\end{eqnarray}}

\newcommand{\lbr}{\left[}
\newcommand{\rbr}{\right]}

\newcommand{\vb}{{\boldsymbol{b}}}

\newcommand{\vp}{{\boldsymbol{p}}}
\newcommand{\vq}{{\boldsymbol{q}}}

\newcommand{\vk}{{\boldsymbol{k}}}
\newcommand{\vm}{{\boldsymbol{m}}}

\newcommand{\vs}{{\boldsymbol{s}}}
\newcommand{\vu}{{\boldsymbol{u}}}
\newcommand{\vv}{{\boldsymbol{v}}}
\newcommand{\vx}{{\boldsymbol{x}}}
\newcommand{\vtheta}{{\boldsymbol{\theta}}}
\newcommand{\vphi}{{\boldsymbol{\phi}}}

\newcommand{\veta}{{\boldsymbol{\eta}}}

\newcommand{\tp}{{\scriptscriptstyle +}}

\newtheorem{theorem}{Theorem}
\newtheorem{definition}{Definition}

\nipsfinalcopy 

\begin{document}

\maketitle

\begin{abstract}
Loopy belief propagation performs approximate inference on graphical models with loops. One might hope to compensate for the approximation by adjusting model parameters. Learning algorithms for this purpose have been explored previously, and the claim has been made that every set of locally consistent marginals can arise from belief propagation run on a graphical model. On the contrary, here we show that many probability distributions have marginals that cannot be reached by belief propagation using any set of model parameters or any learning algorithm. We call such marginals `unbelievable.' This problem occurs whenever the Hessian of the Bethe free energy is not positive-definite at the target marginals. All learning algorithms for belief propagation necessarily fail in these cases, producing beliefs or sets of beliefs that may even be worse than the pre-learning approximation. We then show that averaging inaccurate beliefs, each obtained from belief propagation using model parameters perturbed about some learned mean values, can achieve the unbelievable marginals.
\end{abstract}

\section{Introduction}

Calculating marginal probabilities for a graphical model generally requires summing over exponentially many states, and is NP-hard in general \cite{Cooper:1990p8103}. A variety of approximate methods have been used to circumvent this problem. One popular technique is belief propagation (BP), in particular the sum-product rule, which is a message-passing algorithm for performing inference on a graphical model \cite{Pearl:1988p7576}. Though exact and efficient on trees, it is merely an approximation when applied to graphical models with loops.

A natural question is whether one can compensate for the shortcomings of the approximation by setting the model parameters appropriately. In this paper, we prove that some sets of marginals simply cannot be achieved by belief propagation. For these cases we provide a new algorithm that can achieve much better results by using an ensemble of parameters rather than a single instance.

We are given a set of variables $\vx$ with a given probability distribution $P(\vx)$ of some data. We would like to construct a model that reproduces certain of its marginal probabilities, in particular those over individual variables $p_i(x_i)=\sum_{\vx\setminus x_i}P(\vx)$ for nodes $i\in V$, and those over some relevant clusters of variables, $p_\alpha(\vx_\alpha)=\sum_{\vx\setminus \vx_\alpha}P(\vx)$ for $\alpha=\{i_1,\ldots,i_{d_\alpha}\}$. We will write the collection of all these marginals as a vector $\vp$.

We assume a model distribution $Q_0(\vx)$ in the exponential family taking the form
\be
Q_0(\vx)=e^{-E(\vx)}/Z
\label{eq:BoltzmannDistribution}
\ee
with normalization constant $Z=\sum_\vx e^{-E(\vx)}$ and energy function
\be
E(\vx)=-\sum_\alpha \vtheta_{\alpha} \cdot\vphi_{\alpha}(\vx_\alpha)
\label{eq:EnergyFunction}
\ee
Here, $\alpha$ indexes sets of interacting variables (factors in the factor graph \cite{Kschischang:2001p8492}), and $\vx_\alpha$ is a subset of variables whose interaction is characterized by a vector of sufficient statistics $\vphi_\alpha(\vx_\alpha)$ and corresponding natural parameters $\vtheta_{\alpha}$. We assume without loss of generality that each $\vphi_\alpha(\vx_\alpha)$ is irreducible, meaning that it cannot be written as a sum of any linearly independent functions that themselves do not depend on any $x_i$ for $i\in\alpha$. We collect all these sufficient statistics and natural parameters in the vectors $\vphi$ and $\vtheta$.

Normally when learning a graphical model, one would fit its parameters so the marginal probabilities match the target. Here, however, we will not use {\it exact} inference to compute the marginals. Instead we will use {\it approximate} inference via loopy belief propagation to match the target.

\section{Learning in Belief Propagation}

\subsection{Belief propagation}

The sum-product algorithm for belief propagation on a graphical model with energy function (\ref{eq:EnergyFunction}) uses the following equations \cite{Bishop:2006p8189}:
\be
\label{eq:Messages}
m_{i\to \alpha}(x_i)\propto\!\!\!\prod_{\beta\in N_i\setminus \alpha}\!\!\!m_{\beta\to i}(x_i)\hspace{3em}
m_{\alpha\to i}(x_i)\propto\sum_{\vx_\alpha\setminus x_i}e^{\vtheta_{\alpha}\cdot\vphi_{\alpha}(\vx_\alpha)}\!\!\prod_{j\in N_\alpha\setminus i}m_{j\to\alpha}(x_j)
\ee
where $N_i$ and $N_\alpha$ are the neighbors of node $i$ or factor $\alpha$ in the factor graph. Once these messages converge, the single-node and factor beliefs are given by
\be
\label{eq:PairwiseBeliefs}
b_i(x_i) \propto \prod_{\alpha\in N_i} m_{\alpha\to i}(x_i)\hspace{3em}
b_{\alpha}(\vx_\alpha) \propto e^{\vtheta_{\alpha}\cdot\vphi_{\alpha}(\vx_\alpha)}\!\!\prod_{i\in N_\alpha}\!\!m_{i\to \alpha}(x_i)
\ee
where the beliefs must each be normalized to one. For tree graphs, these beliefs exactly equal the marginals of the graphical model $Q_0(\vx)$. For loopy graphs, the beliefs at fixed points are often good approximations of the marginals. While they are guaranteed to be locally consistent, $\sum_{\vx_\alpha\setminus x_i}b_{\alpha}(\vx_\alpha)=b_i(x_i)$, they are not necessarily globally consistent: There may not exist a single joint distribution $B(\vx)$ of which the beliefs are the marginals \cite{JWainwright:2008p7968}. This is why the resultant beliefs are called {\it pseudomarginals}, rather than simply marginals. We use a vector $\vb$ to refer to the set of both node and factor beliefs produced by belief propagation.

\subsection{Bethe free energy}

Despite its limitations, BP is found empirically to work well in many circumstances. Some theoretical justification for loopy belief propagation emerged with proofs that its stable fixed points are local minima of the Bethe free energy \cite{Yedidia:2001p6930,Heskes:2003p6866}. Free energies are important quantities in machine learning because the Kullback-Leibler divergence between the data and model distributions can be expressed in terms of free energies, so models can be optimized by minimizing free energies appropriately.

Given an energy function $E(\vx)$ from (\ref{eq:EnergyFunction}), the Gibbs free energy of a distribution $Q(\vx)$ is
\be
\label{eq:FreeEnergy}
F[Q]=U[Q]-S[Q]
\ee
where $U$ is the average energy of the distribution
\be
U[Q]=\sum_\vx E(\vx)Q(\vx)=-\sum_\alpha\vtheta_\alpha\cdot\sum_{\vx_\alpha}\vphi_\alpha(\vx_\alpha)q_\alpha(\vx_\alpha)
\ee
which depends on the marginals $q_\alpha(\vx_\alpha)$ of $Q(\vx)$, and $S$ is the entropy
\be
S[Q]=-\sum_\vx Q(\vx)\log{Q(\vx)}
\ee
Minimizing the Gibbs free energy $F[Q]$ recovers the distribution $Q_0(\vx)$ for the graphical model (\ref{eq:BoltzmannDistribution}).

The Bethe free energy $F^\beta$ is an approximation to the Gibbs free energy,
\be
\label{eq:BetheFreeEnergy}
F^\beta[Q]=U[Q]-S^\beta[Q]
\ee
in which the average energy $U$ is exact, but the true entropy $S$ is replaced by an approximation, the Bethe entropy $S^\beta$, which is a sum over the factor and node entropies \cite{Yedidia:2001p6930}:
\be
\label{eq:BetheEntropyWhole}
S^\beta[Q]=\sum_{\alpha} S_\alpha[q_\alpha]+\sum_{i} (1-d_i)S_i[q_i]
\ee
\be
S_\alpha[q_\alpha]=-\sum_{\vx_\alpha}q_\alpha(\vx_\alpha)\log{q_\alpha(\vx_\alpha)}\hspace{3em}
S_i[q_i]=-\sum_{x_i}q_i(x_i)\log{q_i(x_i)}
\ee
The coefficients $d_i=|N_i|$ are the number of factors neighboring node $i$, and compensate for the overcounting of single-node marginals due to overlapping factor marginals. For tree-structured graphical models, which factorize as $Q(\vx)=\prod_{\alpha}q_\alpha(\vx_\alpha)\prod_{i}q_i(x_i)^{1-d_i}$, the Bethe entropy is exact, and hence so is the Bethe free energy. On loopy graphs, the Bethe entropy $S^\beta$ isn't really even an entropy ({\it e.g.} it may be negative) because it neglects all statistical dependencies other than those present in the factor marginals. Nonetheless, the Bethe free energy is often close enough to the Gibbs free energy that its minima approximate the true marginals \cite{Welling:2001p7064}. Since stable fixed points of BP are minima of the Bethe free energy \cite{Yedidia:2001p6930,Heskes:2003p6866}, this helped explain why belief propagation is often so successful.

To emphasize that the Bethe free energy directly depends only on the marginals and not the joint distribution, we will write $F^\beta[\vq]$ where $\vq$ is a vector of pseudomarginals $q_\alpha(\vx_\alpha)$ for all $\alpha$ and all $\vx_\alpha$. Pseudomarginal space is the convex set \cite{JWainwright:2008p7968} of all $\vq$ that satisfy the positivity and local consistency constraints,
\be
\label{eq:LocalConsistency}
0\leq q_\alpha(\vx_\alpha)\leq 1\hspace{3em}
\sum_{\vx_\alpha\setminus x_i}q_\alpha(\vx_\alpha)=q_i(x_i)\hspace{3em}
\sum_{x_i} q_i(x_i)=1
\ee

\subsection{Pseudo-moment matching}
\label{sec:PseudoMomentMatching}

We now wish to correct for the deficiencies of belief propagation by identifying the parameters $\vtheta$ so that BP produces beliefs $\vb$ matching the true marginals $\vp$ of the target distribution $P(\vx)$. Since the fixed points of BP are stationary points of $F^\beta$ \cite{Yedidia:2001p6930}, one may simply try to find parameters $\vtheta$ that produce a stationary point in pseudomarginal space at $\vp$, which is a necessary condition for BP to reach a fixed point there. Simply evaluate the gradient at $\vp$, set it to zero, and solve for $\vtheta$.

Note that in principle this gradient could be used to directly minimize the Bethe free energy, but $F^\beta[\vq]$ is a complicated function of $\vq$ that usually cannot be minimized analytically \cite{Welling:2001p7064}. In contrast, here we are using it to solve for the parameters needed to move beliefs to a target location. This is much easier, since the Bethe free energy is linear in $\vtheta$. This approach to learning parameters has been described as `pseudo-moment matching' \cite{Wainwright:2003p7171, Welling:2003p7058, welling2005}.

The $L_\vq$-element vector $\vq$ is an overcomplete representation of the pseudomarginals because it must obey the local consistency constraints (\ref{eq:LocalConsistency}). It is convenient to express the pseudomarginals in terms of a minimal set of parameters $\veta$ with the smaller dimensionality $L_\vtheta$ as $\vtheta$ and $\vphi$, using an affine transform
\be
\label{eq:MinimalRepresentation}
\vq=W\veta+\vk
\ee
where $W$ is an $L_\vq\times L_\vtheta$ rectangular matrix. One example is the expectation parameters $\veta_\alpha=\sum_{\vx_\alpha} q_\alpha(\vx_\alpha)\vphi_\alpha(\vx_\alpha)$ \cite{JWainwright:2008p7968}, giving the energy simply as $U=-\vtheta\cdot\veta$. The gradient with respect to those minimal parameters is
\begin{align}
\label{eq:BetheGradient}
\frac{\partial F^\beta}{\partial\veta}&=\frac{\partial U}{\partial\veta}-\frac{\partial S^\beta}{\partial\vq}\frac{\partial \vq}{\partial\veta}=-\vtheta-\frac{\partial S^\beta}{\partial\vq}W
\end{align}
The Bethe entropy gradient is simplest in the overcomplete representation $\vq$,
\be
\label{eq:BetheGradientOvercomplete}
\frac{\partial S^\beta}{\partial q_\alpha(\vx_\alpha)}=-1-\log{q_\alpha(\vx_\alpha)}\hspace{3em}
\frac{\partial S^\beta}{\partial q_i(x_i)}=(-1-\log{q_i(x_i)})(1-d_i)
\ee
Setting the gradient (\ref{eq:BetheGradient}) to zero, we have a simple linear equation for the parameters $\vtheta$ that tilt the Bethe free energy surface (Figure \ref{fig:BetheFreeEnergy}A) enough to place a stationary point at the desired marginals $\vp$:
\be
\label{eq:PseudoMomentMatching}
\vtheta=-\left.\frac{\partial S^\beta}{\partial\vq}\right|_{\vp}W
\ee

\subsection{Unbelievable marginals}
\label{sec:Unbelievable}

It is well known that BP may converge on fixed points that cannot be realized as marginals of any joint distribution. In this section we show that the converse is also true: There are some distributions whose marginals cannot be realized as beliefs for any set of couplings. In these cases, existing methods for learning often yield poor results, sometimes even worse than performing no learning at all. This is surprising in view of claims to the contrary: \cite{Wainwright:2003p7171, JWainwright:2008p7968} state that belief propagation run after pseudo-moment matching can always reach a fixed point that reproduces the target marginals. While BP does technically have such fixed points, they are not always stable and thus may not be reachable by running belief propagation.

\begin{definition}
A set of marginals are `unbelievable' if belief propagation cannot converge to them for any set of parameters.
\end{definition}

For belief propagation to converge to the target --- namely, the marginals $\vp$ --- a zero gradient is not sufficient: The Bethe free energy must also be a local minimum \cite{Heskes:2003p6866}.\footnote{Even this is not sufficient, but it is necessary.} This requires a positive-definite Hessian of $F^\beta$ (the `Bethe Hessian' $H$) in the subspace of pseudomarginals that satisfies the local consistency constraints. Since the energy $U$ is linear in the pseudomarginals, the Hessian is given by the second derivative of the Bethe entropy,
\be
\label{eq:BetheHessian}
H=\frac{\partial^2F^\beta}{\partial\veta^2}=-W^\top\frac{\partial^2S^\beta}{\partial\vq^2}W
\ee
where projection by $W$ constrains the derivatives to the subspace spanned by the minimal parameters $\veta$. If this Hessian is positive definite when evaluated at $\vp$ then the parameters $\vtheta$ given by (\ref{eq:PseudoMomentMatching}) give $F^\beta$ a minimum at the target $\vp$. If not, then the target cannot be a stable fixed point of loopy belief propagation. In Section \ref{sec:Experiments}, we calculate the Bethe Hessian explicitly for a binary model with pairwise interactions.

\begin{figure}
\centering
\includegraphics*[width=5.5in]{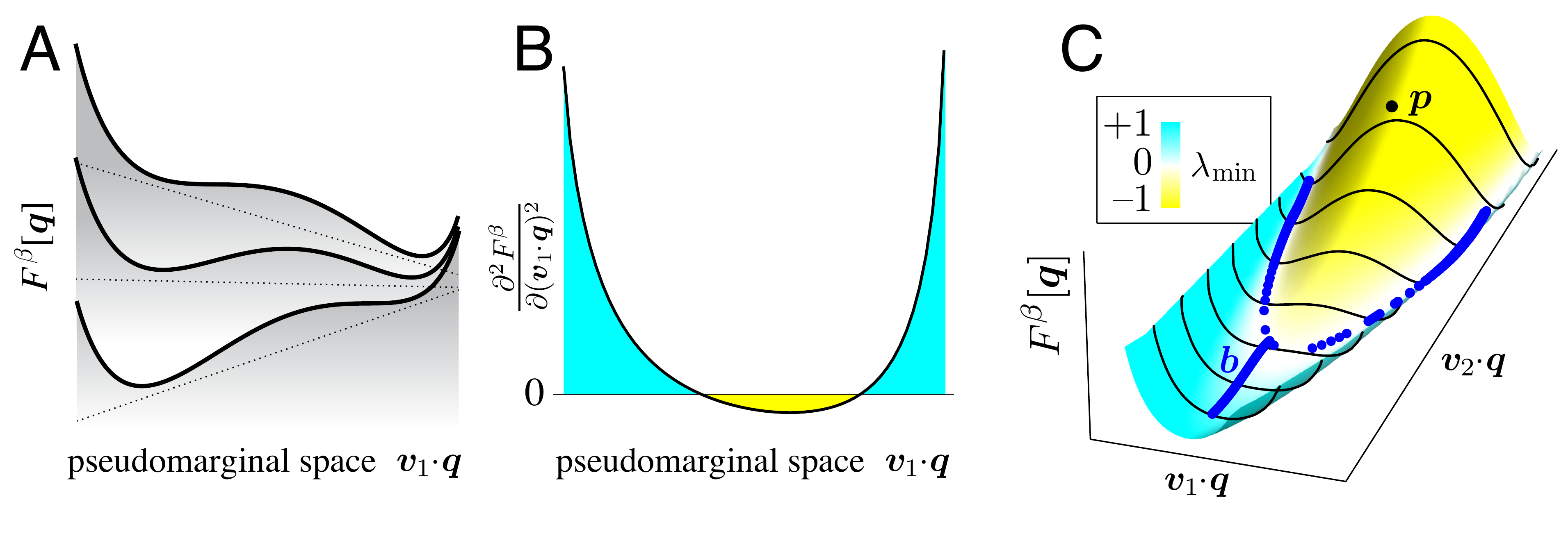}
\caption{Landscape of Bethe free energy for the binary graphical model with pairwise interactions. ({\bf A}) A slice through the Bethe free energy (solid lines) along one axis $\vv_1$ of pseudomarginal space, for three different values of parameters $\vtheta$. The energy $U$ is linear in the pseudomarginals (dotted lines), so varying the parameters only changes the tilt of the free energy. This can add or remove local minima. ({\bf B}) The second derivatives of the free energies in (A) are all identical. Where the second derivative is positive, a local minimum can exist (cyan); where it is negative (yellow), no parameters can produce a local minimum. ({\bf C}) A two-dimensional slice of the Bethe free energy, colored according to the minimum eigenvalue $\lambda_{\rm min}$ of the Bethe Hessian. During a run of Bethe wake-sleep learning, the beliefs (blue dots) proceed along $\vv_2$ toward the target marginals $\vp$. Stable fixed points of BP can exist only in the believable region (cyan), but the target $\vp$ resides in an unbelievable region (yellow). As learning equilibrates, the fixed points jump between believable regions on either side of the unbelievable zone.}
\label{fig:BetheFreeEnergy}
\end{figure}

\begin{theorem}
\label{thm:UnbelievableMarginals}
Unbelievable marginal probabilities exist.
\end{theorem}
\begin{proof}
Proof by example. The simplest unbelievable example is a binary graphical model with pairwise interactions between four nodes, $\vx\in\{-1,+1\}^4$, and the energy $E(\vx)=-J\sum_{(ij)}x_ix_j$. By symmetry and (\ref{eq:BoltzmannDistribution}), marginals of this target $P(\vx)$ are the same for all nodes and pairs: $p_i(x_i)=\frac{1}{2}$ and $p_{ij}(x_i=x_j)=\rho=(2+4/(1+e^{2J}-e^{4J}+e^{6J}))^{-1}$. Substituting these marginals into the appropriate Bethe Hessian (\ref{eq:BetheHessianIsing}) gives a matrix that has a negative eigenvalue for all $\rho>\frac{3}{8}$, or $J>0.316$. The associated eigenvector $\vu$ has the same symmetry as the marginals, with single-node components $u_i=\frac{1}{2}(-2+7\rho-8\rho^2+\sqrt{10-28\rho+81\rho^2-112\rho^3+64\rho^4})$ and pairwise components $u_{ij}=1$. Thus the Bethe free energy does not have a minimum at the marginals of these $P(\vx)$. Stable fixed points of BP occur only at local minima of the Bethe free energy \cite{Heskes:2003p6866}, and so BP cannot reproduce the marginals $\vp$ for any parameters. Hence these marginals are unbelievable.
\end{proof}

Not only do unbelievable marginals exist, but they are actually quite common, as we will see in Section \ref{sec:Experiments}. Graphical models with multinomial or gaussian variables and at least two loops always have some pseudomarginals for which the Hessian is not positive definite \cite{Watanabe:2011p8171}. On the other hand, all marginals with sufficiently small correlations are believable because they are guaranteed to have a positive-definite Bethe Hessian \cite{Watanabe:2011p8171}. Stronger conditions have not yet been described.

\subsection{Bethe wake-sleep algorithm}

When pseudo-moment matching fails to reproduce unbelievable marginals, an alternative is to use a gradient descent procedure for learning, analagous to the wake-sleep algorithm used to train Boltzmann machines \cite{Hinton:1983p8072}. The original rule can be derived as gradient descent of the Kullback-Leibler divergence between the target $P(\vx)$ and the graphical model $Q(\vx)$ (\ref{eq:BoltzmannDistribution}),
\be
D_{\rm KL}[P||Q]=\sum_\vs P(\vx)\log{\frac{P(\vx)}{Q(\vx)}}=F[P]-F[Q]
\ee
where $F$ is the Gibbs free energy (\ref{eq:FreeEnergy}) using the energy function (\ref{eq:EnergyFunction}). Here we use a new cost function, the `Bethe divergence' $D_\beta[\vp||\vb]$, by replacing these free energies by Bethe free energies \cite{welling:sutton:aistats2005} evaluated at the true marginals $\vp$ and at the beliefs $\vb$ obtained from BP fixed points,
\be
\label{eq:CostFunction}
D_\beta[\vp||\vb]=F^\beta[\vp]-F^\beta[\vb]
\ee
We use gradient descent to optimize this cost, with gradient
\be
\frac{dD_\beta}{d\vtheta}=\frac{\partial D_\beta}{\partial \vtheta}+\frac{\partial D_\beta}{\partial \vb}\frac{\partial \vb}{\partial \vtheta}
\ee
The data's free energy does not depend on the beliefs, so $\partial F^\beta[\vp]/\partial \vb=0$, and fixed points of belief propagation are stationary points of the Bethe free energy, so $\partial F^\beta[\vb]/\partial \vb=0$. Consequently $\partial D_\beta/\partial\vb=0$. Furthermore, the entropy terms of the free energies do not depend explicitly on $\vtheta$, so
\be
\label{eq:BetheWakeSleep}
\frac{dD_\beta}{d\vtheta}=\frac{\partial U(\vp)}{\partial \vtheta} -\frac{\partial U(\vb)}{\partial \vtheta}=-\veta(\vp)+\veta(\vb)
\ee
where $\veta(\vq)=\sum_{\vx}q(\vx)\vphi(\vx)$ are the expectations of the sufficient statistics $\vphi(\vx)$ under the pseudomarginals $\vq$. This gradient forms the basis of a simple learning algorithm. At each step in learning, belief propagation is run, obtaining beliefs $\vb$ for the current parameters $\vtheta$. The parameters are then changed in the opposite direction of the gradient,
\be
\label{eq:DeltaCost}
\Delta\vtheta=-\epsilon\frac{dD_\beta}{d\vtheta}=\epsilon(\veta(\vp)-\veta(\vb))
\ee
where $\epsilon$ is a learning rate. This generally increases the Bethe free energy for the beliefs while decreasing that of the data, hopefully allowing BP to draw closer to the data marginals. We call this learning rule the Bethe wake-sleep algorithm.

Within this algorithm, there is still the freedom of how to choose initial messages for BP at each learning iteration. The result depends on these initial conditions because BP can have several stable fixed points. One might re-initialize the messages to a fixed starting point for each run of BP, choose random initial messages for each run, or restart the messages where they stopped on the previous learning step. In our experiments we use the first approach, initializing to constant messages at the beginning of each BP run.

The Bethe wake-sleep learning rule sometimes places a minimum of $F^\beta$ at the true data distribution, such that belief propagation can give the true marginals as one of its (possibly multiple) fixed points. However, for the reasons provided above, this cannot occur where the Bethe Hessian is not positive definite.

\subsection{Ensemble belief propagation}

When the Bethe wake-sleep algorithm attempts to learn unbelievable marginals, the parameters and beliefs do not reach a fixed point but instead continue to vary over time (Figure \ref{fig:Averaging}A,B). Still, if learning reaches equilibrium, then the temporal average of beliefs is equal to the unbelievable marginals.

\begin{theorem}
If the Bethe wake-sleep algorithm reaches equilibrium, then unbelievable marginals are matched by the belief propagation fixed points averaged over the equilibrium ensemble of parameters.
\end{theorem}
\begin{proof}
At equilibrium, the time average of the parameter changes is zero by definition, $\langle\Delta\vtheta\rangle_t=0$. Substitution of the Bethe wake-sleep equation, $\Delta\vtheta=\epsilon(\veta(\vp)-\veta(\vb(t)))$ (\ref{eq:BetheWakeSleep}), directly implies that $\langle\veta(\vb(t))\rangle_t=\veta(\vp)$. The deterministic mapping (\ref{eq:MinimalRepresentation}) from the minimal representation to the pseudomarginals gives $\la \vb(t)\ra_t=\vp$.
\end{proof}

After learning has equilibrated, fixed points of belief propagation occur with just the right frequency so that they can be averaged together to reproduce the target distribution exactly (Figure \ref{fig:Averaging}C). Note that none of the individual fixed points may be close to the true marginals. We call this inference algorithm {\it ensemble} belief propagation (eBP).

Ensemble BP produces perfect marginals by exploiting a constant, small amplitude learning, and thus assumes that the correct marginals are perpetually available. Yet it also works well when learning is turned off, if parameters are drawn randomly from a gaussian distribution with mean and covariance matched to the equilibrium distribution, $\vtheta\sim\mathcal{N}(\bar{\vtheta},\Sigma_\vtheta)$. In the simulations below (Figures \ref{fig:Averaging}C--D, \ref{fig:Performance}B--C), $\Sigma_\vtheta$ was always low-rank, and only one or two principle components were needed for good performance. The gaussian ensemble is not quite as accurate as continued learning (Figure \ref{fig:Performance}B,C), but the performance is still markedly better than any of the available fixed points.

If the target is not within a convex hull of believable pseudomarginals, then learning cannot reach equilibrium: Eventually BP gets as close as it can but there remains a consistent difference $\veta(\vp)-\veta(\vb)$, so $\vtheta$ must increase without bound. Though possible in principle, we did not observe this effect in any of our experiments. There may also be no equilibrium if belief propagation at each learning iteration fails to converge.

\begin{figure}
\centering
\includegraphics*[width=5.5in]{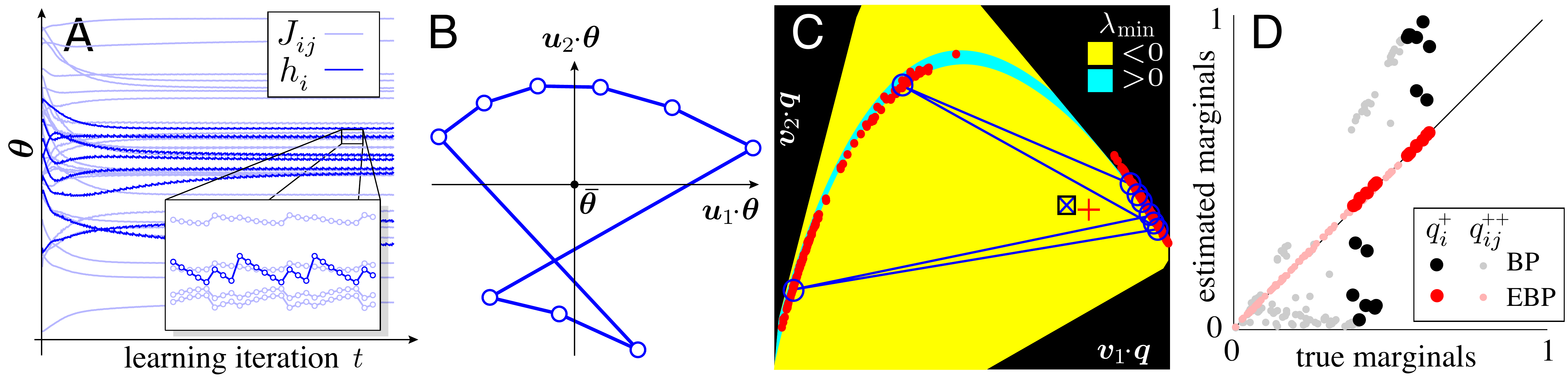}
\caption{Averaging over variable couplings can produce marginals otherwise unreachable by belief propagation. ({\bf A}) As learning proceeds, the Bethe wake-sleep algorithm causes parameters $\vtheta$ to converge on a discrete limit cycle when attempting to learn unbelievable marginals. ({\bf B}) The same limit cycle, projected onto their first two principal components $\vu_1$ and $\vu_2$ of $\vtheta$ during the cycle. ({\bf C}) The corresponding beliefs $\vb$ during the limit cycle (blue circles), projected onto the first two principal components $\vv_1$ and $\vv_2$ of the trajectory through pseudomarginal space. Believable regions of pseudomarginal space are colored with cyan and the unbelievable regions with yellow, and inconsistent pseudomarginals are black. Over the limit cycle, the average beliefs $\bar{\vb}$ (blue $\times$) are precisely equal to the target marginals $\vp$ (black $\square$). The average $\bar{\vb}$ (red $+$) over many fixed points of BP (red dots) generated from randomly perturbed parameters $\bar{\vtheta}+\delta\vtheta$ still produces a better approximation of the target marginals than any of the individual believable fixed points. ({\bf D}) Even the best amongst several BP fixed points cannot match unbelievable marginals (black and grey). Ensemble BP leads to much improved performance (red and pink).}
\label{fig:Averaging}
\end{figure}

\section{Experiments}
\label{sec:Experiments}

The experiments in this section concentrate on the Ising model: $N$ binary variables, $\vs\in\{-1,+1\}^N$, with factors comprising individual variables $x_i$ and pairs $x_i$, $x_j$. The energy function is $E(\vx)=-\sum_i h_i x_i-\sum_{(ij)} J_{ij}x_ix_j$. Then the sufficient statistics are the various first and second moments, $x_i$ and $x_ix_j$, and the natural parameters are $h_i$, $J_{ij}$. We use this model both for the target distributions and the model.

We parameterize pseudomarginals as $\{q_i^\tp,q_{ij}^{\tp\tp}\}$ where $q_i^\tp=q_i(x_i=+1)$ and $q_{ij}^{\tp\tp}=q_{ij}(x_i=x_j=+1)$ \cite{Welling:2001p7064}. The remaining probabilities are linear functions of these values. Positivity constraints and local consistency constraints then appear as $0\leq q_i^+\leq 1$ and $\max(0,q_i^++q_j^+-1)\leq q_{ij}^{++}\leq\min(q_i^+,q_j^+)$. If all the interactions are finite, then the inequality constraints are not active \cite{Yedidia:2005p7837}. In this parameterization, the elements of the Bethe Hessian (\ref{eq:BetheHessian}) are
\begin{subequations}
\label{eq:BetheHessianIsing}
\begin{align}
-\frac{\partial^2S^\beta}{\partial q_i^\tp \partial q_j^\tp}=&
\ \delta_{i,j}(1-d_i)\lbr(q_i^\tp)^{-1}+(1-q_i^\tp)^{-1}\rbr
+\delta_{j\in N_i}\lbr(1-q_i^\tp-q_j^\tp+q_{ij}^{\tp\tp})^{-1}\rbr\\
\nonumber &+\delta_{i,j}\sum_{k\in N_i}\lbr(q_i^\tp-q_{ik}^{\tp\tp})^{-1}+(1-q_i^\tp-q_k^\tp+q_{ik}^{\tp\tp})^{-1}\rbr\\
-\frac{\partial^2S^\beta}{\partial q_{i}^\tp \partial q_{jk}^{\tp\tp}}=
&-\delta_{i,j}\lbr(q_i^\tp-q_{ik}^{\tp\tp})^{-1}+(1-q_i^\tp-q_k^\tp+q_{ik}^{\tp\tp})^{-1}\rbr\\
\nonumber&-\delta_{i,k}\lbr(q_i^\tp-q_{ij}^{\tp\tp})^{-1}+(1-q_i^\tp-q_j^\tp+q_{ij}^{\tp\tp})^{-1}\rbr\\
-\frac{\partial^2S^\beta}{\partial q_{ij}^{\tp\tp} \partial q_{k\ell}^{\tp\tp}}=
&\ \delta_{ij,k\ell}\lbr(q_{ij}^{\tp\tp})^{-1}+(q_i^\tp-q_{ij}^{\tp\tp})^{-1}+(q_j^\tp-q_{ij}^{\tp\tp})^{-1}+(1-q_i^\tp-q_j^\tp+q_{ij}^{\tp\tp})^{-1}\rbr
\end{align}
\end{subequations}

Figure \ref{fig:Performance}A shows the fraction of marginals that are unbelievable for 8-node, fully-connected Ising models with random coupling parameters $h_i\sim\mathcal{N}(0,\frac{1}{3})$ and $J_{ij}\sim\mathcal{N}(0,\sigma_J)$. For $\sigma_J\gtrsim\frac{1}{4}$, most marginals cannot be reproduced by belief propagation with any parameters, because the Bethe Hessian (\ref{eq:BetheHessianIsing}) has a negative eigenvalue.

\begin{figure}
\centering
\includegraphics*[width=5.5in]{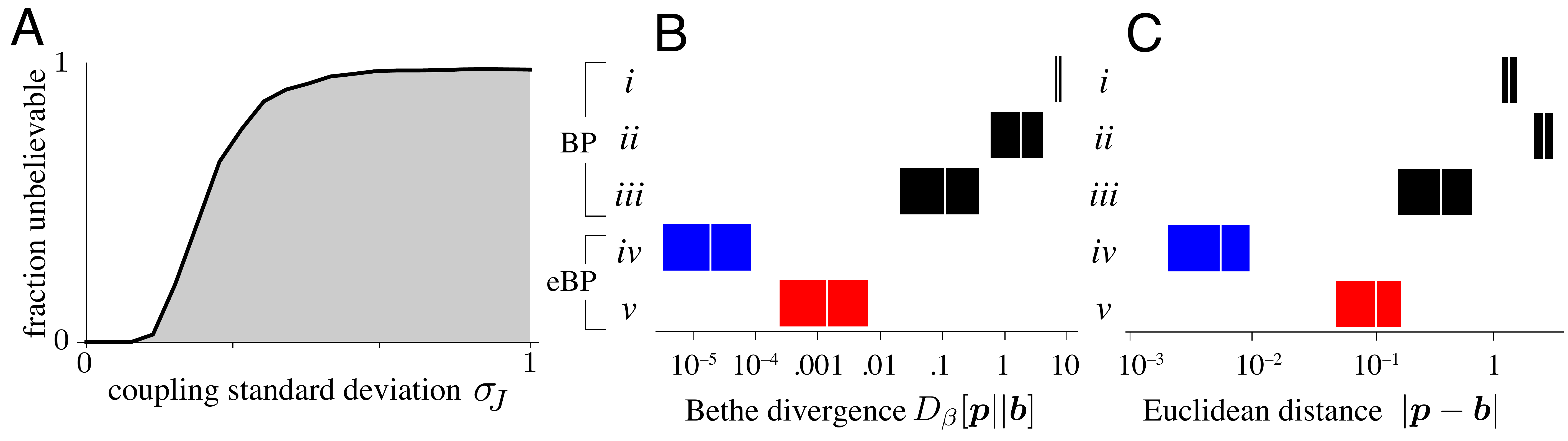}
\caption{Performance in learning unbelievable marginals. ({\bf A}) Fraction of marginals that are unbelievable. Marginals were generated from fully connected, 8-node binary models with random biases and pairwise couplings, $h_i\sim\mathcal{N}(0,\frac{1}{3})$ and $J_{ij}\sim\mathcal{N}(0,\sigma_J)$. ({\bf B},{\bf C}) Performance of five models on 370 unbelievable random target marginals (Section \ref{sec:Experiments}), measured with Bethe divergence $D_\beta[\vp||\vb]$ (B) and Euclidean distance $|\vp-\vb|$ (C). Target were generated as in (A) with $\sigma_J=\frac{1}{3}$, and selected for unbelievability. Bars represent central quartiles, and white line indicates the median. The five models are: ({\it i}) BP on the graphical model that generated the target distribution, ({\it ii}) BP after parameters are set by pseudomoment matching, ({\it iii}) the beliefs with the best performance encountered during Bethe wake-sleep learning, ({\it iv}) eBP using exact parameters from the last 100 iterations of learning, and ({\it v}) eBP with gaussian-distributed parameters with the same first- and second-order statistics as {\it iv}.}
\label{fig:Performance}
\end{figure}

We generated 500 Ising model targets using $\sigma_J=\frac{1}{3}$, selected the unbelievable ones, and evaluated the performance of BP and ensemble BP for various methods of choosing parameters $\vtheta$. Each run of BP used exponential temporal message damping of 5 time steps \cite{Mooij:2005p7495}, $\vm^{t+1}=a\vm^{t}+(1-a)\vm_{\rm undamped}$ with $a=e^{-1/5}$. Fixed points were declared when messages changed by less than $10^{-9}$ on a single time step. We evaluated BP performance for the actual parameters that generated the target (\ref{eq:BoltzmannDistribution}), pseudomoment matching (\ref{eq:PseudoMomentMatching}), and at best-matching beliefs obtained at any time during Bethe wake-sleep learning. We also measured eBP performance for two parameter ensembles: the last 100 iterations of Bethe wake-sleep learning, and parameters sampled from a gaussian $\mathcal{N}(\bar{\vtheta},\Sigma_\vtheta)$ with the same mean and covariance as that ensemble.

Belief propagation gave a poor approximation of the target marginals, as expected for a model with many strong loops. Even with learning, BP could never get the correct marginals, which was guaranteed by selection of unbelievable targets. Yet ensemble belief propagation gave excellent results. Using the exact parameter ensemble gave orders of magnitude improvement, limited by the number of beliefs being averaged. The gaussian parameter ensemble also did much better than even the best results of BP.

\section{Discussion}

Other studies have also made use of the Bethe Hessian to draw conclusions about belief propagation. For instance, the Hessian reveals that the Ising model's paramagnetic state becomes unstable in BP for large enough couplings \cite{Mooij:2005p6759}. For another example, when the Hessian is positive definite throughout pseudomarginal space, then the Bethe free energy is convex and thus BP has a unique fixed point \cite{Heskes:2004p6878}. Yet the stronger interpretation appears to be underappreciated: When the Hessian is not positive definite for some pseudomarginals, then BP can never have a fixed point there, for any parameters.

One might hope that by adjusting the parameters of belief propagation in some systematic way, $\vtheta\to\vtheta_{\rm BP}$, one could fix the approximation and so perform exact inference. In this paper we proved that this is a futile hope, because belief propagation simply can never converge to certain marginals. However, we also provided an algorithm that does work: Ensemble belief propagation uses BP on several different parameters with different fixed points and averages the results. This approach preserves the locality and scalability which make BP so popular, but corrects for some of its defects at the cost of running the algorithm a few times. Additionally, it raises the possibility that a systematic compensation for the flaws of BP might exist, but only as a mapping from individual parameters to an ensemble of parameters $\vtheta\to\{\vtheta_{\rm eBP}\}$ that could be used in eBP.

An especially clear application of eBP is to discriminative models like Conditional Random Fields \cite{Lafferty:2001p8531}. These models are trained so that known inputs produce known inferences, and then generalize to draw novel inferences from novel inputs. When belief propagation is used during learning, then the model will fail even on known training examples if they happen to be unbelievable. Overall performance will suffer. Ensemble BP can remedy those training failures and thus allow better performance and more reliable generalization.

This paper addressed learning in fully-observed models only, where marginals for all variables were available during training. Yet unbelievable marginals exist for models with hidden variables as well. Ensemble BP should work as in the fully-observed case, but training will require inference over the hidden variables during both wake and sleep phases.

One important inference engine is the brain. When inference is hard, neural computations may resort to approximations, perhaps including belief propagation \cite{Litvak:2009p8230,Steimer:2009p8213,Ott:2007p8273,Shon:2005p8336,George:2009p8403}. It would be undesirable for neural circuits to have big blind spots, {\it i.e.} reasonable inferences it cannot draw, yet that is precisely what occurs in BP. By averaging over models with eBP, this blind spot can be eliminated. In the brain, synaptic weights fluctuate due to a variety of mechanisms. Perhaps such fluctuations allow averaging over models and thereby reach conclusions unattainable by a deterministic mechanism.

\subsubsection*{Acknowledgments}
The authors thank
Greg Wayne
for helpful conversations.


\small{
\bibliographystyle{plos}
\bibliography{UnbelievableMarginals_3}
}

\end{document}